\documentclass{article}
\usepackage[utf8]{inputenc}
\usepackage{style}

\definecolor{mydarkblue}{rgb}{0,0.08,0.45}
\hypersetup{ %
colorlinks=true,
linkcolor=mydarkblue,
citecolor=mydarkblue,
filecolor=mydarkblue,
urlcolor=mydarkblue,
}

\def\shownotes{1}  \ifnum\shownotes=1
\usepackage{todonotes}
\else
\usepackage[disable]{todonotes}
\fi

\usepackage{algpseudocode}
\usepackage{algorithm}

\title{Turing-Universal Learners with Optimal Scaling Laws}
\author{Preetum Nakkiran
\vspace{5pt}
\\
Hal{\i}c{\i}o{\u g}lu Data Science Institute\\
University of California San Diego\\
\texttt{preetum@ucsd.edu}
}
\date{}

\newcommand{\wt}{\widetilde}
\newcommand{\tO}{\widetilde{\mathcal{O}}}

\begin{document}

\maketitle
\begin{abstract}
For a given distribution, learning algorithm, and performance metric, the \emph{rate}
of convergence (or \emph{data-scaling law}) is the asymptotic behavior of the algorithm's test performance
as a function of number of train samples.
Many learning methods in both theory and practice have power-law rates, i.e. performance scales as $n^{-\alpha}$ for some $\alpha > 0$.
Moreover, both theoreticians and practitioners are concerned with improving the rates of their learning algorithms under settings of interest.

We observe the existence of a ``universal learner,'' which achieves the best possible distribution-dependent 
asymptotic rate among all learning algorithms within a specified runtime (e.g. $\cO(n^2)$), while incurring only polylogarithmic slowdown over this runtime.
This algorithm is uniform, and does not depend on the distribution, and yet achieves best-possible rates for all distributions.

The construction itself is a simple extension of Levin's universal search \citep{levin-search}.
And much like universal search, the universal learner is not at all practical,
and is primarily of theoretical and philosophical interest.
\end{abstract}
\section{Introduction}
\vspace{-5pt}
We consider learning problems, wherein we are given $n$ i.i.d. samples from an unknown distribution $\cD$,
and we want to produce an output that captures some aspect of $\cD$.
In general, we have some notion of a ``loss function'' by which we rank outputs:
in supervised classification for example, this is the test error.
Our goal is, given samples from $\cD$, produce an output with small loss.
This framing includes problems from supervised learning, supervised classification, and statistical estimation.
A key quantity in these settings is the \emph{learning curve} $G(n)$,
which for a given leaning algorithm and distribution, measures
how quickly the loss decays as a function of number of train samples $n$.
That is, the learning curve measures how the performance of a learning algorithm improves with number of samples.
The \emph{rate} of an algorithm on a given distribution quantifies the asymptotic behavior of its learning curve.
A central problem in the theory of learning is understanding optimal rates in settings of theoretical interest.
Likewise, a central problem in the practice of learning is improving the rates in settings of practical interest.
There has been a resurgence of practical interest in this recently, in the study of ``scaling laws'' 
\citep{hestness2017deep,rosenfeld2019constructive,kaplan2020scaling,henighan2020scaling}.

In this note, we observe that from the perspective of asymptotic rates, algorithm design is a solved problem:
there is a single, explicit learning algorithm which achieves essentially the best possible asymptotic rate
for all distributions, among all algorithms within a specified runtime bound.
For concreteness, we describe this in the special case of \emph{supervised classification}.
We give a learning algorithm $\cA^*$ with the following guarantees.
We specify a runtime bound, such as $T(n)=n^2$ in this example. Our algorithm $\cA^*$
will run in time $\wt{\cO}(n^2)$.
For a distribution $\cD$, suppose there exists some other unknown algorithm $\cA$
that runs in time $\cO(n^2)$, and that achieves learning curve $G(n)$ on this distribution.
Then, the learning curve $G^*(n)$ of our algorithm on the same distribution satisfies the following.
\begin{align}
\label{eqn:intro}
\text{For large enough $n$:} \quad
G^*(n) \leq G(0.99n) + 15\sqrt{\frac{\log{\log n}}{n}}.
\end{align}
In particular, if $G(n)$ satisfies a inverse-polynomial scaling law: $G(n) \leq \cO(n^{-\alpha})$ for some $\alpha < 1/2$,
then our algorithm will satisfy a scaling law at least as good: $G^*(n) \leq \cO(n^{-\alpha})$.
We call $\cA^*$ a ``Turing-universal learner,'' since it does not depend on the distribution, and yet achieves 
the same rate as the best-possible distribution-dependent time-bounded learner for any target distribution.
It is also possible to give a related algorithm $A^+$ which does not require
specifying a time-bound. Roughly speaking, the algorithm $A^+$ will run continuously, and after
$\wt{\cO}(n)$ steps it will output the best predictor among all $O(n)$-time algorithms.
After $\wt{\cO}(n^2)$ steps it will output the best predictor among all $O(n^2)$-time algorithms, and so on.

The ideas involved here are not at all new.
The construction is a simple modification of Levin's universal search \citep{levin-search},
and further inspired by Levin's universal one-way-function \citep{levin-owf}.
The observation is simply that Levin's search can apply not only to problems where solutions can be efficiently verified
(e.g. problems in NP),
but also problems where the utility of solutions can be efficiently estimated from samples.
These ideas are also related to Solomonoff induction \citep{solomonoff1964formal,solomonoff1978complexity}
and Hutter's AIXI \citep{hutter2000theory,Hutter:04uaibook}.
We believe our specific asymptotic perspective and computational efficiency considerations
may be of interest, and were unable to find these implications explicitly in prior work, so we detail them here.

The observations in this note are primarily of philosophical and theoretical interest.
The learning algorithms presented are unlikely to ever be useful in practice.
However, they may be interesting to contextualize progress in algorithm design,
and as a counterexample to certain incorrect intuitions about ``No-Free-Lunch'' theorems.
Indeed, at first glance it appears that from the perspective of asymptotic rates and computable learners, we have a free lunch. We remark on this apparent inconsistency below.

\paragraph{Proof Overview.}
We first informally describe the construction.
Fix some time-constructable runtime bound $T(n)$. Algorithm $\cA_T^*$ is parameterized by $T(n)$,
and operates as follows.

Algorithm $\cA_T^*$:
\begin{itemize}
    \item Input: $n$ samples $(z_1, z_2, \dots, z_n)$.
    \item Enumerate the first $\log{n}$ Turing Machines: $M_1, M_2, \dots M_{\log{n}}$
    \item For each TM: Run $M_i$ on the first $0.99n$ samples for at most $T(n)$ steps,
    and record the predictor $f_i$ it outputs (if any).
    Use the last $0.01n$ samples to estimate the test error of $f_i$.
    \item Output the best predictor $f_i$ among the $\log{n}$ predictors, according to their estimated test errors.
\end{itemize}

The idea is, for large enough $n$, this algorithm will
simulate the operation of any other learning algorithm, including the ``optimal'' one.
For example, if we have a certain learning algorithm $M_k$ in mind,
which is the $k$-th Turing Machine in lexicographic order,
then once $n > \exp(k)$, our algorithm $\cA^*$ will simulate $M_k$
as part of its operation, and thus will perform at least as well.
Since we only test a relatively small number of predictors,
a small holdout set is sufficient to estimate all of their errors.
This idea can be extended to work without a specified time-bound (yielding Algorithm $\cA^+$),
by running the Turing machine simulations in parallel, and keeping track of the best predictor at each time.

\paragraph{Remarks on Real Machines.}
For the reader more amenable to real machines than Turing Machines,
a concrete way of thinking about the algorithm is as follows.
On inputs of size $n$, enumerate all the binary strings of length at most $\log\log{n}$.
Interpret these strings as machine code for a learning algorithm, and execute each one inside a VM for at most $T(n)$ steps.
At the end, use a small holdout set to estimate the performance of each algorithm, and output the predictor of the best one.
For large enough $n$, this algorithm will encounter the machine code of the ``best possible'' learning algorithm,
and will perform at least as well.
To be extremely concrete, at large enough $n$ this algorithm will encounter the machine code for, and perform as well as:
\begin{itemize}
    \item A ResNet-50 \citep{he2016deep} trained on the input samples.
    \item A ResNet-50 pretrained on ImageNet, then fine-tuned on the input samples.
    \item The exact weights of GPT-3 \citep{brown2020language}, fine-tuned on the input samples.
    \item Any (non-quantum) learning method that
    may be discovered in the future (say, at NeurIPS 2035).
\end{itemize}

\paragraph{Limitations.}
Nevertheless, we stress that the universal learning algorithm is far from practical.
The catch is in the clause ``for large enough $n$'' in Equation~\eqref{eqn:intro}.
In practice, the algorithm will be useless in a ``transient phase'' for small values of $n$,
and will only start to do well for extremely large values of $n$.
We note that such a ``transient phase'' is actually observed in the scaling laws of real neural networks,
but this parallel is almost certainly purely coincidental.

\paragraph{Remarks on No-Free-Lunch.}
The Turing-universal learner, at first glance, appears to violate No-Free-Lunch (NFL) theorems, since it is a single algorithm
that is asymptotically optimal for all distributions.
In contrast, NFL theorems formalize various ways in which a ``universal learner'' is theoretically impossible (e.g. \citet{wolpert1996lack,wolpert1997no,devroye1982any,shalev2014understanding}).
However, this is not a contradiction, as can be seen by paying careful attention to (1) the order of quantifiers and (2) computability considerations.
Most NFL theorems, roughly speaking, prove the following: 
For any candidate algorithm $\cA$, and any \emph{fixed} sample size $n$, there
exists a distribution $\cD$ and a prediction function $f'$ such that
$f'$ performs (much) better on $\cD$ than the output of $\cA$ on $n$ samples.
That is, there is no ``universal'' algorithm $\cA$ that outputs the best predictor for all distributions.
Our Turing-universal learner avoids this obstacle in two ways: first,
we switch the order of quantifiers.
It is possible that, for every $n$, there is ``bad'' distribution for our algorithm, and
a corresponding a better predictor $f'$ that beats our algorithm.
However, if we fix a distribution and fix some computable predictor $f'$, then 
for large enough $n$ our algorithm will eventually do at least as well as $f'$ (within $\tO(n^{-1/2})$ error).
The second key to avoiding NFL barriers is computability considerations: instead
of comparing our algorithm against arbitrary functions as ``best possible learners'',
we only compare against all \emph{computable} learning algorithms (which, in particular, have finite description length).
This is a reasonable restriction, since allowing otherwise would be a
stretch of the notion of ``learning algorithm.''

\paragraph{Related Works.}
The asymptotic rate of estimators has long been studied in theory,
for example in nonparametric statistics
(e.g. \citet{tsybakov2009introduction}).
Here, a typical result is result is that the rate of convergence
of a particular estimator
is $\cO(n^{-\alpha})$ for some $\alpha < 1/2$ that depends on properties of the distribution
(e.g. smoothness).
Similar results also hold for minimax rates over natural families of distributions.
Note that our Turing-universal learner can be seen to achieve optimal minimax rates
over the class of all \emph{computable} estimators.

In parametric statistics and agnostic PAC learning, the optimal rate of
convergence is $\cO(n^{-1/2})$, where the constant depends on the model or hypothesis class (e.g. VC dimension).
In this context, our Turing-universal learner can be thought of as
Structural Risk Minimization over the hypothesis class of all Turing Machines.
Rates of convergence are also studied in many other learning contexts, such as
online optimization and bandits \citep{bubeck2011introduction}, which are less directly related to our setting.

The study of learning curves in machine learning broadly also has a long history
(e.g. \citet{cover1967nearest,atlas1990training,amari,perlich2003tree,sompolinsky1990learning}), which is surveyed
in \citet{viering2021shape}.
This recently gained further interest, with the observation that large neural networks exhibit power-law learning curves in many realistic settings~\citep{hestness2017deep,kaplan2020scaling,henighan2020scaling,rosenfeld2019constructive}.
The works of \citet{bahri2021explaining,spigler2020asymptotic,bordelon2020spectrum}
explore how this can occur in certain kernel models,
and \citet{hutter2021learning} shows other simplified settings with power-law scaling.

We also acknowledge \citet{bousquet2021theory}, which studies optimal
distribution-dependent rates in realizable PAC learning.
We similarly consider distribution-dependent rates, though our learning setting differs.

\section{Formal Statements}

\subsection{Definitions}
For simplicity of presentation, we will specialize to the case of \emph{supervised classification}
throughout this note, though many of the claims hold more generally.
In supervised classification, we have some input space $\cX$ and label space $\cY$.
We are given $n$ i.i.d. samples from an unknown distribution $\cD$ over $\cZ := (\cX \x \cY)$,
which we call the ``train set,''
and we want to learn a predictor $f: \cX \to \cY$.
Predictors are evaluated by their \emph{on-distribution test error}
$\cL_\cD(f) := \E_{x, y \sim \cD}[ \1\{f(x) \neq y\}]$, and our goal
is to produce predictors with small test error.

\emph{Learning algorithms} are (uniform) Turing Machines
which input samples $(z_1, z_2, \dots, z_n) \in \cZ^n$, and output predictors $f: \cX \to \cY$,
where $f$ is encoded as a Turing Machine.
For notational convenience, we assume that the runtime of the predictor is at most the runtime of
the learning algorithm which produced it.
For a given learning algorithm $\cA$ and a given distribution $\cD$, the
\emph{learning curve}
is the expected test loss as a function of number of train samples:
$$
G_{\cA, \cD}(n) := \E_{S \sim \cD^n}[ \cL_\cD( \cA(S) )].
$$

\subsection{Statements}
We consider the following algorithm, which we call a ``Turing-universal learner.''

\begin{algorithm}
\caption{Turing-Universal Learner $\cA^*_T$}\label{alg:universal}
\begin{algorithmic}[1]
\Require Samples $Z = \{z_1, z_2, \dots z_n\}$.
\Ensure Predictor $\hat{f}$, as encoding of a Turing Machine.

\Procedure{$\cA^*_T$}{$Z$}
\State Let $M_1, M_2, \dots$ be an enumeration of Turing Machines in lexicographic order.
\State Partition the input samples $Z$ into sets $Z_{tr}, Z_{te}$
where $|Z_{tr}|=0.99n$, and $|Z_{te}| = 0.01n$.
\For{$i = 1$ to $\log{n}$}
    \State Let $f_i \gets M_i(Z_{tr})$, halting after at most $T(n)$ steps.
    \State $\hat{\ell_i} \gets \frac{1}{0.01n} \sum_{(x, y) \in Z_{te}} \1\{f_i(x) \neq y\}$ 
    \Comment{Estimate the loss $\cL_\cD(f_i)$ using samples $Z_{te}$}
\EndFor
\State $j \gets \argmin_i \hat{\ell_i}$ \Comment{Return the $f_i$ with best estimated loss}
\State \Return $f_{j}$
\EndProcedure
\end{algorithmic}
\end{algorithm}

The main lemma is as follows.
\begin{lemma}[Universal Learner]
\label{lem:main}
Let $M_1, M_2, \dots$ be an enumeration of Turing Machines in lexicographic order.
For all time-constructable functions $T(n)$,
and all classification domains\footnote{We omit certain routine details about the domain, such admitting efficient encoding via Turing Machines.} $\cX \x \cY$
there exists a Turing machine $\cA_T^*$ with the following properties.
\begin{enumerate}
    \item For all distributions $\cD$ over $\cX \x \cY$, the following holds.
    On training sets of size $n$, the TM $\cA^*_T$ outputs
a predictor $f$ such that, in expectation over train samples,
\begin{align}
\E_{Z \sim \cD^n}[ \cL(\cA^*_T(Z)) ] \leq
\min_{i \leq \log{n}} ~
\E_{Z \sim \cD^{0.99n}} [\cL(\bar{M_i}(Z))] + 10\sqrt{2}\sqrt{\frac{\log{\log n} + \log 2}{n}}
\end{align}
where $\bar{M_i}$ denotes the output of $M_i$ when run for at most $T(n)$ steps.
That is, the expected loss of $\cA^*_T$ matches the ``best possible'' learner
among the first $\log{n}$ turing machines,
when run on $0.99n$ samples.

\item The runtime of $\cA_T^*$ is at most $\tO(T(n))$.
That is, there is only polylogarithmic slowdown over the runtimes of $\bar{M_i}$.
\end{enumerate}
\end{lemma}
The proof is straightforward, and provided in Appendix~\ref{app:proofs}.

\begin{remark}
The additive $\sqrt{\log\log n/n}$ factor in Lemma~\ref{lem:main} can be improved
to any any arbitrarily slow growing function that is $\omega(n^{-1/2})$,
simply by modifying Algorithm~\ref{alg:universal} to only enumerate the first $k(n)$ Turing machines,
for some arbitrarily-slow function $k(n)$.
\end{remark}

\begin{remark}
We chose to specialize to the case of classification, but the algorithm and analysis naturally generalizes
to any setting where the loss function that can be efficiently estimated from samples.
This includes losses which are expectations of bounded functions, but can potentially be even more general.
\end{remark}

This lemma directly implies the following theorem, which states that the asymptotic rate of $A^*_T$ is optimal in a certain regime.
\begin{theorem}[Universal Learner Scaling]
\label{thm:main}
For any time-constructable function $T(n)$, 
and any classification domain $\cX \x \cY$ there exists an algorithm $\cA^*_T$
with the following properties.
Let $M^*$ be any learning algorithm which runs in time $T(n)$, and which outputs predictors that themselves run in time $T(n)$.
Let $\cD$ be any distribution over $\cX \x \cY$.
Suppose the asymptotic test error of $M^*$ is bounded by some polynomial scaling-law
with rate $\alpha$.
That is, for some $\alpha > 0$, the following holds.
$$
\exists N_0 \in \N \forall n \geq N_0 :
\quad
\E_{f \gets M^*(\cD^n)}[ \cL(f) ] \leq C n^{-\alpha}
$$
Then, the asymptotic scaling of the algorithm $\cA^*_T$ is bounded by:
$$
\exists N_1 \in \N \forall n \geq N_1 :
\quad
\E_{f \gets \cA^*_T(\cD^n)}[ \cL(f) ] \leq C(0.99n)^{-\alpha} + \tO(n^{-1/2})
$$
In particular, if $\alpha < 1/2$, then $\cA^*_T$ matches the scaling exponent $\alpha$,
and nearly matches the constant $C$:
$$
\exists N_1 \in \N \forall n \geq N_1 :
\quad
\E_{f \gets \cA^*_T(\cD^n)}[ \cL(f) ] \leq 1.01 Cn^{-\alpha} + o(n^{-\alpha})
$$
Moreover, the algorithm $\cA^*_T$ runs in time $\tO(T(n))$.
\end{theorem}

Note that the constants $N_0, N_1$ in the above theorem will in general be different.
In particular, if $M^*=M_k$ is the $k$-th Turing Machine, then the conclusion of Theorem~\ref{thm:main}
holds for $N_1 = \max(N_0, \exp(k))$.

It is also possible to avoid specifying a time-bound: we can have an algorithm that runs continuously,
and is equivalent to $A^*_T$ if we halt it after $T$ steps.
Thus, roughly speaking, this algorithm is as good as the best $\cO(n)$-time algorithm if halted after $\tO(n)$ steps,
or as good as the best $\cO(n^2)$-time algorithm if halted after $\tO(n^2)$ steps, and so on.
This is captured by the following variant of Lemma~\ref{lem:main}.

\begin{lemma}[Continuous Universal Learner]
\label{lem:main2}
Let $M_1, M_2, \dots$ be an enumeration of Turing Machines in lexicographic order.
For all time-constructable functions $T(n)$,
and all classification domain $\cX \x \cY$
there exists a Turing machine $\cA^+$ such that
for all distributions $\cD$ over $\cX \x \cY$, the following holds.
    
The machine $\cA^+$ runs continuously, and we decide when to halt it and observe its output.
Suppose we halt it at time $T$. Then its output at time $T$, denoted $A^+_T(Z)$ obeys the following.

\begin{align}
\E_{Z \sim \cD^n}[ \cL(\cA^+_T(Z)) ] \leq
\min_{i \leq \log{n}} ~
\E_{Z \sim \cD^{0.99n}} [\cL(\bar{M_i}(Z))] + 10\sqrt{2}\sqrt{\frac{\log{\log n} + \log 2}{n}}
\end{align}
where $\bar{M_i}$ denotes the output of $M_i$ when run for at most $T/\log{n}$ steps.
\end{lemma}

\begin{proof}[Proof Sketch]
The machine $A^+$ is a simple modification of Algorithm~\ref{alg:universal}.
Instead of simulating each of the first $\log{n}$ Turing machines in sequence,
we simulate them all ``in parallel,'' sharing time equally.
At each time step, we overwrite the output tape with the current-best predictor (according our estimate).
When this machine is halted after $T$ steps of parallel simulation, its output is equivalent to the algorithm $A^*_T$,
which explicitly halts all simulations after $T$ steps.
\end{proof}

\section{Conclusions}
We have observed that a simple modification of Levin's universal search~\citep{levin-search}
yields a learning algorithm with asymptotically optimal rates, among all computable learners,
in a wide range of parameters.
The algorithm itself is more of a theoretical curiosity than a practical tool.
In particular, it strongly abuses the fact that asymptotic properties only need to hold for ``large enough $n$''---
and it is indeed poorly behaved for small $n$.
This example is thus a demonstration of the pitfalls of purely asymptotic analysis.
Nevertheless, it contradicts certain incorrect intuitions about ``No Free Lunch'' theorems,
and it may help contextualize progress in learning (theory and practice).
\emph{It is easy to obtain optimal asymptotic rates, so how else should we quantify
algorithmic progress in learning?}

\subsubsection*{Acknowledgements}
We thank many friends and colleagues for their unyielding occupation with scaling laws,
which made this note inevitable.
We thank Boaz Barak, Jaros\l aw B\l asiok, and Fred Zhang for comments on an early draft, which improved the presentation and results. Any errors are my own.

We acknowledge the support of the NSF and the Simons Foundation for the Collaboration on the Theoretical Foundations of Deep Learning through awards DMS-2031883 and \#814639.

\bibliographystyle{plainnat}
\bibliography{refs}

\appendix

\section{Proofs}
\label{app:proofs}

The following fact is standard.
\begin{fact}
\label{fact:max}
Let $X_1, X_2, \dots X_n$ be random variables (not necessarily independent or identically-distributed),
which are all sub-Gaussian with variance parameter $\sigma^2$.
Then:
\[
\E[\max_i X_i] \leq \sigma \sqrt{2 \log{n}}
\]
\end{fact}
\begin{proof}
For any $t \in \R$, we have:
\begin{align}
\exp(t \E[\max_i X_i]) 
&\leq
\E[\exp(t \max_i X_i) ]\tag{Jensen's}\\
&= \E[\max_i \exp(t X_i) ]\\
&\leq \E[\sum_i \exp(t X_i) ]\\
&= \sum_i \E[ \exp(t X_i) ]\\
&\leq \sum_i \exp(t^2 \sigma^2 / 2) \tag{sub-Gaussian}\\
&\leq n \exp(t^2 \sigma^2 / 2)
\end{align}
Setting $t := \sqrt{2 \log n} / \sigma$ recovers the desired result.
\end{proof}

\subsection{Proof of Lemma~\ref{lem:main}}
\begin{proof}[Proof of Lemma~\ref{lem:main}]
Let $f_i \gets M_i(Z_{tr})$ as in Algorithm~\ref{alg:universal}.
Let $\ell_i := \cL_\cD(f_i)$ be the true losses of $f_i$, and
let $\hat{\ell}_i$ be the estimated losses, for all $i \in k=\log{n}$.
Define the random variable
$$\eps := \max_{i \in [k]} |\ell_i - \hat{\ell}_i|$$
Since each $\hat{\ell}_i$ is an estimate of $\ell_i$ using $0.01n$ Bernoulli samples,
the random variables $(\ell_i - \hat{\ell}_i)$ are individually sub-Gaussian with variance parameter $1/(0.04n)$.
We can then apply a standard fact (Fact~\ref{fact:max}) about the maximum of (not necessarily independent) variables to find that
\begin{align}
\label{eqn:eps}
\E[\eps] 
\leq 
\frac{5\sqrt{2 \log 2k}}{\sqrt{n}}
= \frac{5\sqrt{2 \log (2\log n)}}{\sqrt{n}}
\end{align}

Now we will establish the following pointwise property:
\begin{align}
\label{eqn:ptwise}
\cL(\cA_T^*(Z)) \leq \min_{i \leq \log{n}} 
\cL(M_i(Z_{tr})))] + 2\eps
\end{align}

Let $i^* := \argmin_{i \in [\log{n}]} \ell_i$, and $j$ as in Algorithm~\ref{alg:universal}.
We have 
\begin{align*}
\ell_j &= \ell_{i^*}
+ (\hat{\ell}_{i^*} -  \ell_{i^*})
+ (\ell_{j} -  \hat{\ell}_{j})
+ (\hat{\ell}_{j} -  \hat{\ell}_{i^*})\\
&\leq 
\ell_{i^*} 
+ \eps + \eps
+ (\hat{\ell}_{j} -  \hat{\ell}_{i^*})\\
&\leq 
\ell_{i^*} + 2\eps \tag{since $\hat{\ell}_j = \min_i \hat{\ell}_i$}
\end{align*}
This establishes Equation~\eqref{eqn:ptwise}.
Finally, taking expectations we have:
\begin{align*}
\E_{Z\sim \cD^n} [\cL(\cA_T^*(Z))]
&\leq \E[ \min_{i \leq \log{n}} \cL(M_i(Z_{tr})) + 2\eps]
\tag{by Equation~\eqref{eqn:ptwise}}\\
&= \E[ \min_{i \leq \log{n}} \cL(M_i(Z_{tr}))] + 2\E[\eps]\\
&\leq \min_{i \leq \log{n}} \E[ \cL(M_i(Z_{tr}))] + 2\E[\eps]\\
&\leq \min_{i \leq \log{n}} \E[ \cL(M_i(Z_{tr}))] + \frac{10\sqrt{2 \log (2\log n)}}{\sqrt{n}}
\tag{by Equation~\ref{eqn:eps}}
\end{align*}
as desired.
\end{proof}

\end{document}